\documentclass[runningheads]{llncs}
\usepackage[T1]{fontenc}
\usepackage{multirow,multicol,bm,tikz,pgfplots,booktabs,url,hyperref,bm,comment,amsmath,svg,subcaption,float,mathtools,amssymb}
\usepackage[most]{tcolorbox}
\usepackage{xcolor}
\tcbuselibrary{breakable}
\usetikzlibrary{arrows.meta,positioning}
\tikzset{>=Stealth,node distance=2cm,var/.style={font=\large\itshape},every path/.style={thick}}
\usepackage{listings}
\makeatletter
\newcommand*{\da@rightarrow}{\mathchar"0\hexnumber@\symAMSa 4B }
\newcommand*{\da@leftarrow}{\mathchar"0\hexnumber@\symAMSa 4C }
\newcommand*{\xdashrightarrow}[2][]{%
  \mathrel{%
    \mathpalette{\da@xarrow{#1}{#2}{}\da@rightarrow{\,}{}}{}%
  }%
}
\newcommand{\xdashleftarrow}[2][]{%
  \mathrel{%
    \mathpalette{\da@xarrow{#1}{#2}\da@leftarrow{}{}{\,}}{}%
  }%
}
\newcommand{\xdashleftrightarrow}[2][]{%
  \mathrel{%
    \mathpalette{\da@xarrow{#1}{#2}\da@leftarrow\da@rightarrow{\,}{\,}}{}%
  }%
}
\newcommand*{\da@xarrow}[7]{%
  \sbox0{$\ifx#7\scriptstyle\scriptscriptstyle\else\scriptstyle\fi#5#1#6\m@th$}%
  \sbox2{$\ifx#7\scriptstyle\scriptscriptstyle\else\scriptstyle\fi#5#2#6\m@th$}%
  \sbox4{$#7\dabar@\m@th$}%
  \dimen@=\wd0 %
  \ifdim\wd2 >\dimen@
    \dimen@=\wd2 %
  \fi
  \count@=2 %
  \def\da@bars{\dabar@\dabar@}%
  \@whiledim\count@\wd4<\dimen@\do{%
    \advance\count@\@ne
    \expandafter\def\expandafter\da@bars\expandafter{%
      \da@bars
      \dabar@ 
    }%
  }%
  \mathrel{#3}%
  \mathrel{%
    \mathop{\da@bars}\limits
    \ifx\\#1\\%
    \else
      _{\copy0}%
    \fi
    \ifx\\#2\\%
    \else
      ^{\copy2}%
    \fi
  }%
  \mathrel{#4}%
}
\makeatother

\pgfplotsset{compat=1.18}
\begin{document}
\newcommand{\dsep}{\perp\mkern-9.5mu\perp}

\title{Automatic Causal Fairness Analysis with LLM-Generated Reporting}
\author{Alessia Berarducci, Eric Rossetto, Alessandro Antonucci, Marco Zaffalon}
\institute{Istituto Dalle Molle di Studi sull'Intelligenza Artificiale (IDSIA), USI-SUPSI, Lugano, Switzerland \\
\email{\{alessia.berarducci, eric.rossetto, alessandro.antonucci, marco.zaffalon\}@supsi.ch}
}

\maketitle  
\begin{abstract}
AutoML, intended as the process of automating the application of machine learning to real-world problems, is a key step for AI popularisation. Most AutoML frameworks are not accounting for the potential lack of fairness in the training data and in the corresponding predictions. We introduce \textsc{FairMind}, a software prototype aiming to automatise fairness analysis at the dataset level. We achieve that by resorting to the assumptions of the \emph{standard fairness model}, recently proposed by Plečko and Bareinboim. This allows for a sound fairness evaluation in terms of causal effects, based on \emph{counterfactual} queries involving the target, possibly confounders and mediators, and the different values of an input feature we regard as \emph{protected}. After the necessary data preprocessing, the tool implements a closed-form computation of the effects. LLMs are consequently exploited to generate accurate reports on the fairness levels detected in the training dataset. We achieve that in a zero-shot setup and show by examples the expected advantages with respect to a direct analysis performed by the LLM. To favour applications, extensions to ordinal protected variable and continuous targets and novel decomposition results are also discussed.
\end{abstract}
\section{Introduction}
In recent years, significant efforts have been devoted to \emph{automated} machine learning (AutoML), where ML algorithms are trained from data with zero or minimal human preprocessing and directly used to perform predictions \cite{barbudo2023}. The evaluation and, when possible, the mitigation of \emph{fairness} issues also emerged as an important aspect to be considered in the design of these predictive tools \cite{weertsCanFairnessBe2024}. However, to the best of our knowledge, little attention has been paid to fairness within the AutoML community. This situation appears unfortunate, since frameworks such as the causal one proposed by Plečko and Bareinboim \cite{pleckoCausalFairnessAnalysis2024} already provide a powerful and systematic approach to fairness quantification.

A major goal of this paper is therefore to promote such a counterfactual fairness framework in AutoML settings by reviewing the semantics of the key inferences, together with the corresponding identification formulae. These quantities can be obtained from observational data and their computation naturally embedded within the tool we present. Notably, we also extend the framework to a number of more general cases commonly encountered in applications, including ordinal protected variables and continuous target variables. The resulting quantitative analysis is paired with a \emph{large language model} (LLM), that provides automated textual reporting based on a zero-shot paradigm guided by a \emph{chain-of-thoughts} scheme. This is in line with the most recent trends in AutoML, where LLM agents are increasing flexibility and usability \cite{surveyAutoMLLLMs}.

Within the Pearlian causal inference framework \cite{pearlCausalInferenceStatistics2016}, adapted to algorithmic fairness in \cite{pleckoCausalFairnessAnalysis2024}, multiple works have explored fairness from a theoretical perspective. However, these approaches typically rely on manual modelling decisions and expert interpretation, which somehow contrasts with the AutoML philosophy \cite{weertsCanFairnessBe2024}. Instead of proposing new fairness definitions, we integrate existing causal fairness methods into an automated pipeline. We develop a tool that automatically selects, estimates and explains causal fairness effects with minimal user interventions. Result interpretation is supported by an LLM, enabling non-experts to conduct and understand causal fairness analyses within an AutoML workflow. The work in \cite{mahajanCausalFairnessInActionOpenSource2025} also describes an implementation of counterfactual fairness metrics, but it only considers aggregated effects and it is not publicly available. Our framework focuses instead on automating effect selection, interpretation and reporting through LLM-assisted pipelines.
A web interface to use the framework is publicly available.\footnote{\href{https://fairmindcausal.streamlit.app}{fairmindcausal.streamlit.app}.} The code and some demonstrative examples are also available.\footnote{\href{https://github.com/Erhtric/fairmind-causal-fairness-analysis}{github.com/Erhtric/fairmind-causal-fairness-analysis}.}

The paper is organised as follows. Sect.~\ref{sec:background} introduces the basic concepts and notation. The standard fairness model is presented in Sect.~\ref{sec:sfm}, together with the necessary identification formulae. Sect.~\ref{sec:theory_new} presents our extensions of the framework. The automatic reporting based on LLMs is discussed in Sect.~\ref{sec:app-ui-llm} together with the \textsc{FairMind} application we release. Conclusions and outlooks for future work are in Sect.~\ref{sec:conc}. As a supplementary material, we gather the proofs in App.~\ref{app:proofs}, and the prompt we used to generate the fairness reports in App.~\ref{app:prompt}.

\section{Basics}\label{sec:background}
We use uppercase letters for variables, lowercase is used instead for states, and calligraphic for sets of the states. Accordingly, $v \in \mathcal{V}$ is a generic state of $V$.

\paragraph{\bf Structural Causal Models (SCMs).} Following \cite{pearlCausalityModelsReasoning2009}, we define an SCM $\mathcal{M}$ (SCMs) as a tuple $\langle \bm{V}, \bm{U}, \mathcal{F}, P(\bm{U}) \rangle$ where: (i) $\bm{V}$ is a set of \emph{endogenous} (observed) variables; (ii)  $\bm{U}$ is a set of \emph{exogenous} (latent) variables distributed according to $P(\bm{U})$; (iii) $\mathcal{F}$ is a collection of functions, denoted as \emph{structural equations}. Each $V_i \in \bm{V}$ is determined by a function $f_{V_i} \in \mathcal{F}$, defined as a mapping from a set of exogenous variables $\bm{U}_i \subseteq \bm{U}$ and a set of endogenous variables $\mathrm{Pa}_{V_i} \subseteq \bm{V} \setminus \{V_i\}$ to the domain of $V_i$. The structural equations induce a \emph{causal graph} $\mathcal{G}$, where an edge $(X,Y)$ exists if and only if $X \in \mathrm{Pa}_Y$. Throughout the manuscript, we focus our attention on \emph{recursive} SCMs, whose causal graph is acyclic. Bi-directed edges ($V_i \xdashleftrightarrow[]{} V_j$) indicate a common (latent) \emph{confounder} $U \in \bm{U}$ appearing both in $f_{V_i}$ and $f_{V_j}$. In the absence of bi-directed edges, or in other words if there is a one-to-one correspondence between $\bm{U}$ and $\bm{V}$, the SCM is \emph{Markovian}; otherwise, it is \emph{Semi-Markovian}. A \emph{linear SCM} specifies that each structural equation is a linear function.

\paragraph{\bf Interventions, Counterfactuals and Nesting.} In an SCM, $P(\bm{U})$ and $\mathcal{F}$ induce a joint \emph{observational} distribution $P(\bm{V})$ over $\bm{V}$. \emph{Interventions} are the semantic operations that allow us to modify this data generating regime. For an arbitrary set $\bm{X} \subseteq \bm{V}$, an \emph{intervention} yields a new model $\mathcal{M}_{\bm{x}}$ (i.e., a \emph{sub-model}) by replacing the structural equations associated with $\bm{X}$ by constant assignments $\bm{X}\gets\bm{x}$. This operation leaves $P(\bm{U})$ and $\mathcal{F} \setminus \{f_{x}\}_{x\in\bm{X}}$ unchanged, and it is commonly denoted as $do(\bm{X}=\bm{x})$. For any $Y\in\bm{V}$, the impact of an intervention $\bm{X} \gets \bm{x}$ on $Y$ is the \emph{potential response} $Y_{\bm{x}}(\bm{u})$. It denotes the value of $Y$ in sub-model $\mathcal{M}_{\bm{x}}$ given $\bm{U}=\bm{u}$. A \emph{counterfactual} variable (or \emph{potential outcome}) $Y_{\bm{x}}$ is the variable induced by $Y_{\bm{x}}(\bm{u})$ when $\bm{u} \sim P(\bm{U})$. Computing counterfactuals from evidence $\bm{E} = \bm{e}$ follows a three-step procedure: (i) compute $P(\bm{U}\mid \bm{E}=\bm{e})$; (ii) intervene to construct $\mathcal{M}_{\bm{x}}$; and (iii) compute $Y_{\bm{x}}$ \cite{pearlCausalityModelsReasoning2009}. Furthermore, counterfactuals can be \emph{nested} \cite{pearlDirectIndirectEffects2001}: $Y_{x, Z_{x'}}\coloneq f_{Y}(x, Z_{x'}(\bm{u}), \bm{u})$, represents the value $Y$ would take had $X$ been set to $x$, while $Z$ is set to the value it would have attained had $X$ been $x'$.

\paragraph{\bf Identifiability.} We define a causal query $Q$ as the probabilistic evaluation of arbitrary combinations of counterfactual variables given some evidence, i.e., $P(\bm{Y}_{\bm{x}} = \bm{y} \mid \bm{E}=\bm{e})$. For any SCM $\mathcal{M}$, if $P(\bm{U})$ and $\mathcal{F}$ are fully specified, evaluating $Q$ follows the three-step procedure detailed above. However, in practice, $\mathcal{F}$ and $P(\bm{U})$ remain unknown. Instead, one must establish \emph{identifiability}: deciding whether a target $Q$ can be uniquely computed from the observational distribution $P(\bm{V})$ and the structural assumptions encoded in the causal graph $\mathcal{G}$. We typically resolve \emph{(interventional) identifiability}, e.g., $Q = P(Y_{\bm{x}})$, by employing the \emph{do-calculus} \cite{pearlCausalityModelsReasoning2009}. In contrast, evaluating algorithmic fairness queries requires \emph{(counterfactual) identifiability}, which involves determining identifiability for nested counterfactuals or joint distributions under conflicting interventions, e.g., $Q = P(Y_{x'}\mid x,y)$. To address this, the \emph{counterfactual calculus} (ctf-calculus) generalizes \emph{do-calculus} to identify independencies across multiple worlds \cite{correaCounterfactualGraphicalModels2025}. 

\section{Standard Fairness Model}\label{sec:sfm}
Let $X$ and $Y$ denote two endogenous variables of an SCM $\mathcal{M}$. We define $Y$ as the \emph{target} variable and $X$ as a \emph{protected} feature that topologically precedes $Y$ in the causal graph of $\mathcal{M}$. While we typically assume $X$ is a parent of $Y$, our analysis holds even if $X$ is merely an ancestor. We initially assume $X$ is a binary variable with states $x_0$ and $x_1$. The causal effect of $X$ on $Y$ may operate through one or more \emph{mediating variables} $\bm W$, representing intermediate factors along the causal paths from $X$ to $Y$. Furthermore, a set of observed \emph{confounders} $\bm Z \subset \bm V$ may generate spurious, non-causal associations between $X$ and $Y$. Graphically, this confounding is denoted by a bi-directed dashed edge $X \xdashleftrightarrow[]{}  Z$, indicating the presence of shared, unobserved exogenous variables $U$ influencing both $X$ and $Z$. \emph{Adjusting} for $Z$ effectively blocks all non-causal paths, thereby neutralising the confounding bias and satisfying the necessary criteria to identify the causal effect of $X$ on $Y$. Fig.~\ref{fig:causalgraph_simple} illustrates the topology induced by $\mathcal{M}$. This clustered graphical structure corresponds to the \emph{Standard Fairness Model} (SFM)  \cite{pleckoCausalFairnessAnalysis2024}. The SFM represents a template-like graph---an equivalence class of causal diagrams---that alleviates causal modelling requirements whilst preserving the identifiability of causal fairness quantities. It is noteworthy to point out that: (i) it remains agnostic to the internal structural relationships within the possibly multidimensional sets $\bm Z$ and $\bm W$; and (ii) it encodes the assumption of no hidden confounding between the respective variable groups, an assumption whose relaxations have been discussed \cite{pleckoCausalFairnessAnalysis2024}. Throughout this work, we assume the underlying SCM $\mathcal{M}$ compatible with the SFM assumptions.

\begin{figure}[htp!]
\centering
\begin{subfigure}{0.45\textwidth}
\centering
\begin{tikzpicture}[scale=0.9]
\node[var] (X) at (-1.5,0) {$X$};
\node[var] (Y) at ( 3,0) {$Y$};
\node[var] (Z) at ( 0.7,1) {$\bm Z$};
\node[var] (W) at ( 0.7,-1) {$\bm W$};
\draw[->] (X) -- (Y);
\draw[->] (Z) -- (Y);
\draw[->] (X) -- (W);
\draw[->] (W) -- (Y);
\draw[->] (Z) -- (W);
\draw[dashed,<->,bend right=25] (Z) to (X);
\end{tikzpicture}
\caption{}\label{fig:causalgraph_simple}
\end{subfigure}
\hfill
\begin{subfigure}{0.45\textwidth}
\centering
\begin{tikzpicture}[scale=0.9]
\node[var] (X)  at (-1.5,0) {$X$};
\node[var] (Y)  at ( 3,0) {$Y$};
\node[var] (Z1) at ( 0,) {$Z^1$};
\node[var] (Z2) at ( 1.5,1) {$Z^2$};
\node[var] (W1) at ( 0,-1) {$W^1$};
\node[var] (W2) at ( 1.5,-1) {$W^2$};
\draw[->] (X) -- (Y);
\draw[->] (Z1) -- (Y);
\draw[->] (Z2) -- (Y);
\draw[->] (Z1) -- (Z2);
\draw[->] (Z1) -- (W2);
\draw[->] (Z2) -- (W1);
\draw[->] (X) -- (W1);
\draw[->] (X) -- (W2);
\draw[->] (W1) -- (W2);
\draw[->] (W1) -- (Y);
\draw[->] (W2) -- (Y);
\draw[->] (Z1) -- (W1);
\draw[->] (Z2) -- (W2);
\draw[dashed,<->,bend right=25] (Z1) to (X);
\draw[dashed,<->,bend right=45]  (Z2) to (X);
\end{tikzpicture}
\caption{}\label{fig:causalgraph_detailed}
\end{subfigure}
\caption{Causal graph depicting the SFM assumptions (a) and an explicit version with two mediators and two confounders (b).}\label{fig:causalgraph}
\end{figure}
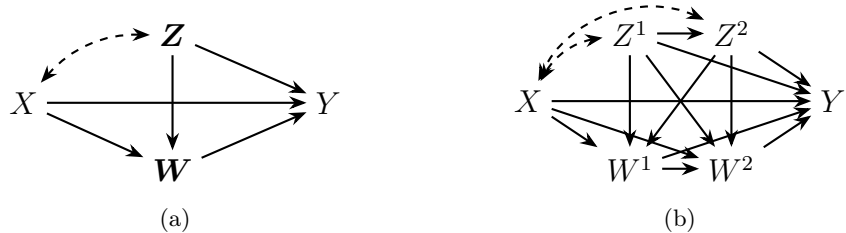

\vspace{-5mm}

\subsection{Causal Effects}
Consider an SCM $\mathcal{M}$ consistent with the SFM assumptions discussed in the previous section and, for the moment, with a single mediator $W$ and a single confounder $Z$. Given a target state $y\in\mathcal{Y}$ and a change in the state of the protected feature $X$ from $x_0$ to $x_1$, consider the following, composite, queries:
\begin{eqnarray}
\mathrm{TV}_{x_0,x_1}(y)\coloneq&&P(y \mid x_1)-P(y \mid x_0)\,,\label{eq:tv}\\
\mathrm{TE}_{x_0,x_1}(y)\coloneq&&P(Y_{x_1}=y)-P(Y_{x_0}=y)\,,\label{eq:te}\\
\mathrm{SE}_{x_0,x_1}(y)\coloneq&&\mathrm{TV}_{x_0,x_1}(y)- \mathrm{TE}_{x_0,x_1}(y)\,,\label{eq:se}\\
\mathrm{DE}_{x_0,x_1}(y)\coloneq&&P(Y_{x_1,W_{x_0}}=y)-P(Y_{x_0,W_{x_0}}=y)\,,\label{eq:de}\\
\mathrm{IE}_{x_0,x_1}(y)\coloneq&&P(Y_{x_0,W_{x_1}}=y)-P(Y_{x_0,W_{x_0}}=y)\,.\label{eq:ie}
\end{eqnarray}

The \emph{total variation} (TV) in Eq.~\eqref{eq:tv} is an observational query denoting the change in the probability of the target state $y$ induced by a change in the (observation of the) protected feature from $x_0$ to $x_1$. The \emph{total effect} (TE) in Eq.~\eqref{eq:te} is the change in the probability of having $Y=y$ when the treatment changes by interventions from $\mathrm{do}(X=x_0)$ to $\mathrm{do}(X=x_1)$. The mediator $W$ does not appear explicitly in the formula: its dependence on $X$ is implicit, as it attains its value \emph{naturally} according to the level of $X$ rather than through direct intervention. Importantly, the TE is defined through an intervention on $X$, rather than conditioning on its observed value, thereby blocking spurious associations. The \emph{spurious effect} (SE) in Eq.~\eqref{eq:se} is the difference between observational and interventional quantities. 

The TE can be further decomposed into components that isolate the direct and mediated causal pathways.
The (natural) \emph{direct effect} (DE) in Eq.~\eqref{eq:de} measures the expected increase in $Y$ from $X=x_0$ to $X=x_1$, with $W$ set to the value it would have taken naturally under $X=x_0$. It captures direct path from $X$ to $Y$ while turning off the mediator’s response to treatment. The (natural) \emph{indirect effect} (IE) in Eq.~\eqref{eq:ie} is evaluated by holding the variable $X$ at $x_0$. It captures the part of the TE passing through the mediator.

Let $\hat{P}(\bm{V})$ denote the endogenous distribution we learn from a dataset $\mathcal{D}$ of endogenous observations. Being an observational query, the value of TV can be obtained trivially obtained from $P(\bm{V})$. Under the SFM, all the effects are also \emph{identifiable}, i.e., they can be obtained from the observational distribution, without any information about the structural equations and the exogenous distributions. This is formally stated by the following result.

\begin{proposition}\label{prop:identification}
The following identification formulae hold:
\begin{eqnarray}
\textrm{TE}_{x_0,x_1}(y)=&&
\sum_{z} \left[ \hat{P}(y \mid x_1, z) - \hat{P}(y \mid x_0, z) \right] \hat{P}(z)\,,\label{eq:te2}\\
\textrm{DE}_{x_0,x_1}(y)=&&
\sum_{z,w} \big[ \hat{P}(y \mid x_1, z, w) - \hat{P}(y \mid x_0, z, w) \big] 
\hat{P}(w \mid x_0, z) \hat{P}(z)\,,\label{eq:de2}\\
\textrm{IE}_{x_0,x_1}(y)=&&
\sum_{z,w} \hat{P}(y \mid x_0, z, w)
\big[ \hat{P}(w \mid x_1, z) - \hat{P}(w \mid x_0, z) \big] \hat{P}(z)\,.\label{eq:ie2}
\end{eqnarray}
\end{proposition}
These formulae have been derived in \cite{pleckoCausalFairnessAnalysis2024}. An identification formula from the SE trivially descends from Eq.~\eqref{eq:se}, Eq.~\eqref{eq:tv}, and Eq.~\eqref{eq:te2}. Analogous identification formulae can likewise be derived for the \emph{specific} effects obtained by further conditioning on $X$ or on $Z$ \cite{pleckoCausalFairnessAnalysis2024}.

\subsection{Effect Decompositions}
We can intend Eq.~\eqref{eq:se} not only as the SE definition, but also as a decomposition formula stating that the TV decomposes in the sum of two terms: the TE modelling the causal component of the relation, and the SE corresponding to correlational effects. A zero SE means that the relation between the two variables is purely causal. For TE, a further decomposition is described in the next proposition. 

\begin{proposition}\label{prop:total_Eff_decomp}
The total effect can be written in terms of direct and indirect effects,  where the indirect effect is taken in its reverse form:
\begin{equation}\label{eq:dec_te}
\mathrm{TE}_{x_0,x_1} (y)=\mathrm{DE}_{x_0,x_1}(y)-\mathrm{IE}_{x_1,x_0}(y)\,.
\end{equation}
Additionally, in case of linear models, we can rewrite Eq.~\eqref{eq:dec_te} as:

\begin{equation}\label{eq:dec_te2}
\mathrm{TE}_{x_0,x_1} (y)=\mathrm{DE}_{x_0,x_1}(y)+\mathrm{IE}_{x_0,x_1}(y)\,.
\end{equation}
\end{proposition}

This decomposition allows us to separate the contributions of the direct and indirect effects in practice, as illustrated in the following example.

\begin{example}\label{ex:adult_te_decomp}
{\it
In the revised version of the UCI \emph{Adult Census} dataset \cite{dingRetiringAdultNew2021}, we consider the following variables: (i) the binary variable \emph{gender}, with states \emph{male} and \emph{female}, which we regard as the protected feature $X$; (ii) a binary variable indicating whether \emph{income} exceeds $\$50k$, regarded as the target variable $Y$; (iii) \emph{education}, a categorical variable with sixteen states describing the education level, which we treat as a confounder $Z$ between $X$ and $Y$; and (iv) the integer feature \emph{hours} (per week), regarded as a mediator $W$ between $X$ and $Y$, which we discretize into 20-hour bins (0--20, $\ldots$, 80+).}
\vspace*{-3ex}
\begin{figure}[htp!]
\centering
\includegraphics[width=9cm]{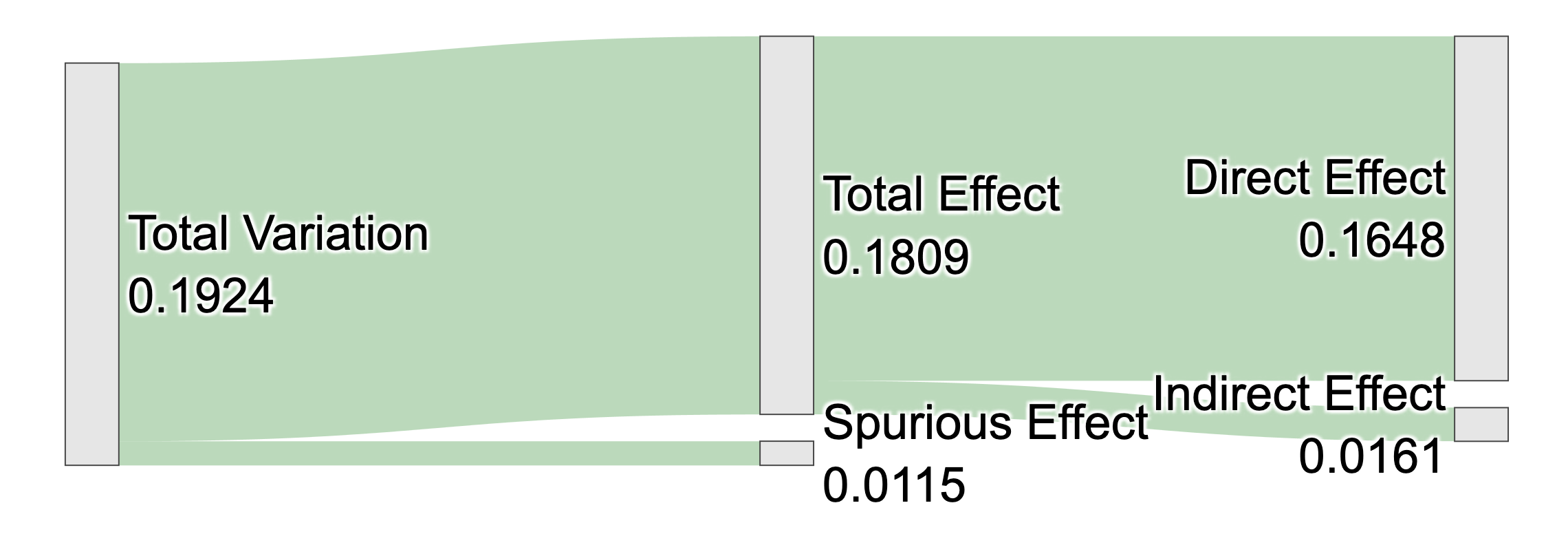}
\caption{Decomposition of the effect of \emph{gender} on \emph{income} for the \emph{Adult} dataset.}
\label{fig:adult_sankey_decomp_ex1}
\end{figure}
\vspace*{-3ex}

{\it We adopt the SFM assumptions and consider an SCM as in Fig.~\ref{fig:causalgraph_simple}. The fairness descriptors in Eqs.~\eqref{eq:tv}--\eqref{eq:ie} are computed using the formulae in Prop.~\ref{prop:identification}, focusing on the probability of low income when $X$ transitions from \emph{male} to \emph{female}. Positive values, therefore, correspond to discrimination against females. For this model, all descriptors are positive, and their corresponding decomposition is illustrated as a Sankey diagram in Fig.~\ref{fig:adult_sankey_decomp_ex1}. Despite a relatively large total variation ($\simeq 19\%$), we observe that the dominant component is causal, with the direct effect accounting for the majority of the disparity ($\simeq 16\%$). Under the adopted causal assumptions, the higher probability of low income among females can be interpreted as being partly explained by the causal effect associated with gender.}
\end{example}

\subsection{More Decompositions: Multiple Confounders and Mediators}
In more general settings, the relationship between the protected feature $X$ and the target $Y$ may be mediated and confounded by multiple variables. This does not preclude the application of the identification formulae discussed above, as the mediators and confounders can simply be treated as joint variables and, when necessary, the expressions can be evaluated by summing over all joint states. However, in \cite{pleckoCausalFairnessAnalysis2024}, decomposition formulae are derived to identify the separate contribution of each mediator to the IE, and of each confounder to the SE. We begin by considering the mediators.

\begin{proposition}\label{prop:ie}
Consider an SCM with $k$ mediators $\bm{W}\coloneqq (W^1,\ldots,W^k)$ sorted in a topological order. For any $1\leq i \leq k$, let $\bm{W}^{\leq i} \coloneq (W^1, \dots, W^i)$ and $\bm{W}^{>i} \coloneq (W^{i+1}, \dots, W^k)$ and set $\bm{W}^{\leq 0} \coloneq \emptyset$. Eq.~\eqref{eq:ie} decomposes as:
\begin{equation}
\mathrm{IE}_{x_0,x_1}(y) = \sum_{i=1}^{k} \mathrm{IE}_{x_0,x_1}^{W^i}(y)\,,
\label{eq:multipleeffW}
\end{equation}
where the components are:
\begin{equation}\label{eq:iew}
\mathrm{IE}_{x_0,x_1}^{W^{i}}(y) \coloneqq P\big(y_{x_0,(\bm{W}^{\leq i})_{x_1},(\bm{W}^{>i})_{x_0}}\big) - P\big(y_{x_0,(\bm{W}^{\leq i-1})_{x_1},(\bm{W}^{>i-1})_{x_0}} \big) \,.
\end{equation}
Moreover, $P(y_{x_0,(\bm{W}^{\leq i})_{x_1},(\bm{W}^{>i})_{x_0}})$ can be identified as:
\begin{equation}\label{eq:id_sfm_multimediator}
\sum_{\bm{w}\in\mathcal{\bm{W}}} \hat{P}(y\mid x_0,\bm{w})
\prod_{j=1}^{i} \hat{P}\left(w^j\mid x_1,\bm{w}^{\leq j-1}\right)\prod_{j=i+1}^{k}\hat{P}\left(w^j\mid x_0,\bm{w}^{\leq j-1}\right)\,,
\end{equation}
where, for each $j$, the values of $\bm{w}^{\leq j-1}$ are those consistent with $\bm{w}$.
\end{proposition}

The above proposition enables a decomposition of the IE given a topological ordering of the mediators. If the underlying graph admits multiple such orderings, a question is whether different choices lead to the same decomposition. A first answer to this question is provided by the following novel result.

\begin{theorem}
Say that the mediators $(W^1,\dots,W^k)$ are conditionally independent given $X$ and $Z$. If the model is also linear, then the decomposition in Eq.~\eqref{eq:multipleeffW} is unaffected by the ordering we adopt.
\label{thm:multiple_mediators}
\end{theorem}

The following result is the analogous of Prop.~\ref{prop:ie} for the SE in the case of a SCM having multiple confounders $\bm{Z}$.

\begin{proposition}\label{prop:spur_more}
In an SFM setting, consider a SCM with $k$ confounders $\bm{Z}\coloneqq (Z^1,\ldots,Z^k)$ already sorted in a topological order. We further assume that there is a one-to-one correspondence between $\bm{U}$ and $\bm{V}$, such that for any distinct pair $V^i, V^j \in \bm{V}$, $(\mathrm{Pa}_{V^{i}} \setminus \bm{V}) \cap (\mathrm{Pa}_{V^{j}} \setminus \bm{V}) = \emptyset$. For any $1\leq i \leq k$, let $\bm{Z}^{\leq i}\coloneq (Z^1, \dots, Z^i)$ and $\bm{Z}^{>i}\coloneq (Z^{i+1}, \dots, Z^k)$ and set $\bm{Z}^{\leq 0} \coloneq \emptyset$. Eq.~\eqref{eq:se} decomposes as:
\begin{equation}\label{eq:multipleZ}
\mathrm{SE}_{x_0,x_1}(y)=\sum_{i=1}^k \mathrm{SE}_{x_0,x_1}^{Z^i}(y)\,,
\end{equation}
where the components can be identified as follows:
\begin{align}\label{eq:sesum}
\mathrm{SE}^{Z^i}_{x_0, x_1}(y) \coloneq \sum_{\bm{z}^{\leq i-1}} P(y \mid x_1, \bm{z}^{\leq i-1}) P(\bm{z}^{\leq i-1}) 
               - \sum_{\bm{z}^{\leq i}} P(y \mid x_1, \bm{z}^{\leq i}) P(\bm{z}^{\leq i}) \\
     -  \sum_{\bm{z}^{\leq i-1}} P(y \mid x_0, \bm{z}^{\leq i-1}) P(\bm{z}^{\leq i-1}) 
               + \sum_{\bm{z}^{\leq i}} P(y \mid x_0, \bm{z}^{\leq i}) P(\bm{z}^{\leq i})\,,\nonumber
\end{align}
for each $i$, where the values of $\bm{z}^{\leq i-1}$ are those consistent with $\bm{z}$.
\end{proposition}

Note that we cannot establish an analogue of Th.~\ref{thm:multiple_mediators} for confounders in place of mediators. The reason is that the SE decomposition in Eq.~\eqref{eq:sesum} is constructed sequentially by conditioning on progressively larger subsets of confounders according to a given ordering.

An example of the above decompositions is the following.

\begin{example}\label{ex:adult_extended}
{\it We aim to quantify the gender-based income disparity for the same dataset as in Ex.~\ref{ex:adult_te_decomp}. We extend the sets of observed variables to include an additional confounder and an additional mediator. The confounders are consequently \emph{relationship} ($Z^1$) and \emph{native country} ($Z^2$), while the mediators are \emph{hours} ($W^1$) and \emph{occupation} ($W^2$). In practice, we work with an SCM whose graph is the one in Fig.~\ref{fig:causalgraph_detailed}. We compute the fairness descriptors defined in Eqs.~\eqref{eq:tv}--\eqref{eq:ie}, and retrieve the individual contributions of the two mediators and the two confounders by applying the decompositions presented in Props.~\ref{prop:ie} and \ref{prop:spur_more}. The results are in Tab.~\ref{tab:adult}. The large disparity ($\mathrm{TV} \simeq 17\%$) is driven by the SE, whose main contribution is the one induced by the \emph{relationship}. The contribution of \emph{native-country} is instead negligible. The total effect is marginal, demonstrating no substantial discrimination. Moreover, most of the causal contribution stems from indirect pathways, where the effect mediated by \emph{occupation} is approximately twice of that mediated by \emph{hours-per-week}, however still not substantive in magnitude.  In conclusion, at the speculative level, one can observe that the macro-level disparity is primarily an artefact of the upstream confounding introduced by \emph{relationship}.}
\end{example}

\begin{table}[htp!]
\centering
\begin{tabular}{lp{1cm}rrrr}
\toprule
Effect&&\multicolumn{3}{c}{Values}&{\footnotesize\color{black!50}{$\%$ of $|\mathrm{TV}|$}}\\
\midrule
$\mathrm{TV}_{x_0, x_1}(y)$&&$0.1736$&&&{\footnotesize\color{black!50}{$100$}}\\
\midrule
$\mathrm{TE}_{x_0,x_1}(y)$&&$0.0053$&&&{\footnotesize\color{black!50}{$3$}}\\
\quad $\mathrm{DE}_{x_0,x_1}(y)$&&&$-0.0051$\\
\quad $\mathrm{IE}_{x_1,x_0}(y)$&&&$-0.0104$\\
\quad $\hookrightarrow$ $\mathrm{IE}^{W^1}_{x_1,x_0}(y)$&&&&$-0.0027$\\
\quad $\hookrightarrow$ $\mathrm{IE}^{W^2}_{x_1,x_0}(y)$&&&&$-0.0077$\\
\midrule
$\mathrm{SE}_{x_0, x_1}(y)$&&$0.1683$&&&{\footnotesize\color{black!50}{$97$}}\\
\quad $\hookrightarrow$ $\mathrm{SE}^{Z^1}_{x_0, x_1}(y)$&&&$0.1685$\\
\quad $\hookrightarrow$ $\mathrm{SE}^{Z^2}_{x_0, x_1}(y)$&&&$-0.0002$\\
\bottomrule
\end{tabular}
\caption{Complete decomposition of the effect of \emph{gender} on \emph{income} for the \emph{Adult} dataset between \emph{male} ($x_0$) and \emph{female} ($x_1$) groups.\label{tab:adult}}
\end{table}

\section{Beyond SFM}\label{sec:theory_new}
\subsection{Non-Binary Targets} The fairness descriptors in Eqs.~\eqref{eq:tv}--\eqref{eq:ie} are defined with respect to a specific target state $y\in\mathcal{Y}$. When $Y$ is binary, it is straightforward to verify that computing the same descriptor for both states yields values with equal magnitude but opposite sign. For example, if $\mathcal{Y}=\{y',y''\}$, then $\mathrm{TV}_{x_0,x_1}(y')=-\mathrm{TV}_{x_0,x_1}(y'')$. This conjugacy no longer holds when $Y$ has three or more states. 

In such cases, one can compute the descriptors separately for each state of $Y$, and then consider an aggregated measure. Distributional dissimilarity measures--such as the widely used KL divergence---could in principle be employed for this purpose. However, these measures fail to capture the \emph{sign} of the fairness violation. For this reason, when $Y$ is numerical, a more informative approach might be to consider  \emph{expectations}, which yield a signed measure of the effect. For instance, for the DE in Eq.~\eqref{eq:de}, one may consider:
\begin{equation}\label{eq:de_exp}
\mathrm{DE}_{x_0,x_1}(Y)\coloneq\mathbb{E}[Y_{x_1,W_{x_0}}]-\mathbb{E}[Y_{x_0,W_{x_0}}]\,.
\end{equation}
Here $\mathbb{E}[V] \coloneq \sum_{v\in\mathcal{V}} v \cdot P(V=v)$. Being linear functionals, coping with expectations as in Eq.~\eqref{eq:de_exp} preserves all the decomposition relationships discussed in the previous section. For the same reason, it is also possible to bypass the computation of the two expectations by directly averaging the effects for the different values of $Y$. Strictly speaking, expectation is defined for numerical variables, as it relies on summation over numerical values. Nevertheless, the same idea can be extended to categorical or ordinal variables by introducing a utility function $T:\mathcal{V}\to\mathbb{R}$. In that case, one simply considers $\mathbb{E}[T(V)]$ in place of $\mathbb{E}[V]$.

\begin{example}\label{ex:student}
{\it In the \emph{Student-Mat} dataset \cite{cortez2008}, let us consider the effect of \emph{gender} on the integer variable (number of) \emph{failures}, with possible values $\mathcal{Y}\coloneq \{0,1,2,3\}$, with \emph{address} acting as a mediator. Tab.~\ref{tab:ternary} depicts the TE and DE for all these four states, together with the expectations computed as in Eq.~\eqref{eq:de_exp}. The results show that the TE and DE are very close across all outcome levels, suggesting that the mediated contribution through \emph{address} is relatively small. The aggregated effect on $Y$ confirms this pattern, with TE and DE taking similar values.}
\end{example}

\begin{table}[htp!]
\centering
\begin{tabular}{lrr}
\toprule
$y$&TE&DE\\
\midrule
0&-0.038&-0.0348\\
1&0.0331&0.0319\\
2&-0.0194&-0.0207\\
3&0.0243&0.0236\\
\midrule
$Y$&0.0672&0.0613\\
\bottomrule
\end{tabular}
\caption{Effects of the four \emph{failure} target states with respect to \emph{gender} and, in the last row, the corresponding expected values.\label{tab:ternary}}
\end{table}

Finally, it is a simple remark that, when coping with a continuous target, we can also consider expectations as in Eq.~\eqref{eq:de_exp}, where sums are replaced by integrals. A similar extension applies to continuous mediators and confounders, in which case the sums in Eqs.~\eqref{eq:te2}, \eqref{eq:de2}, and \eqref{eq:ie2} are likewise replaced by integrals.

Such integrals might rule out relevant behaviours in correspondence of particular values of the target. As an alternative approach, we might discretise the target as a binary variable for different choices of the discretisation bin, thus identify the values of the target maximising an effect, as in the following example.

\begin{example}
{\it For the same dataset as in Ex.~\ref{ex:adult_te_decomp}, let us regard the (working) \emph{hours} per week as a continuous target $Y$, \emph{gender} as a binary protected feature and \emph{occupation} as mediator. Fig.~\ref{fig:cont} shows the values of two effects obtained by binarising for different thresholds. We observe a pronounced effect around the full-time employment cut-off, indicating that individuals in the \emph{female} group are substantially less likely to work more than 40 hours per week compared to \emph{males}. The result indicates that occupational differences reduce, rather than explain, the observed disparity. Overall, these results suggest that disparities in labour supply at the full-time threshold persist even after accounting for occupation.}
\end{example}

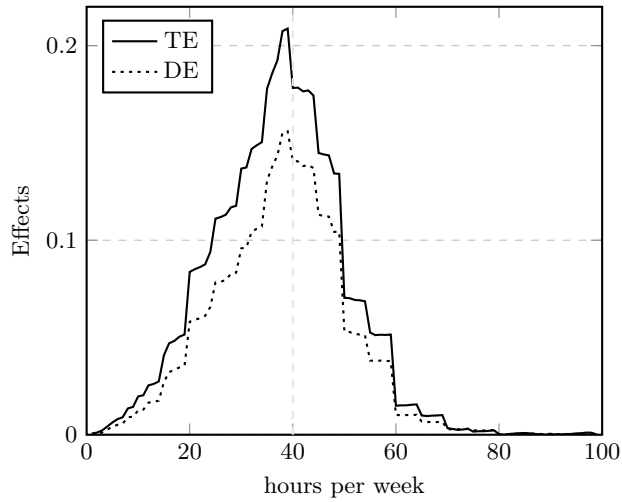
\begin{figure}[htp!]
\centering
\begin{tikzpicture}
\begin{axis}[title={},xlabel={hours per week},ylabel={Effects},xmin=0, xmax=100,ymin=0.0, ymax=0.22,
ytick={0,0.1,0.2},legend pos=north west,ymajorgrids=true,grid style=dashed]
\addplot[color=black,very thick,mark=none]coordinates {(1,0.0006)(2,0.0011)(3,0.0022)(4,0.0042)(5,0.0063)(6,0.008)(7,0.009)(8,0.0135)(9,0.0143)(10,0.0197)(11,0.0203)(12,0.0254)(13,0.0261)(14,0.0274)(15,0.0409)(16,0.047)(17,0.0482)(18,0.0504)(19,0.0515)(20,0.0837)(21,0.0851)(22,0.0862)(23,0.0876)(24,0.094)(25,0.1111)(26,0.1119)(27,0.1131)(28,0.1169)(29,0.1177)(30,0.1368)(31,0.1374)(32,0.1469)(33,0.1488)(34,0.1504)(35,0.1779)(36,0.1858)(37,0.1924)(38,0.2075)(39,0.2088)(40,0.1783)(41,0.1784)(42,0.1765)(43,0.177)(44,0.1744)(45,0.1448)(46,0.1441)(47,0.1436)(48,0.1343)(49,0.1341)(50,0.0704)(51,0.0702)(52,0.0693)(53,0.0691)(54,0.0686)(55,0.0525)(56,0.0512)(57,0.0514)(58,0.0513)(59,0.0515)(60,0.0149)(61,0.0151)(62,0.0151)(63,0.0154)(64,0.0156)(65,0.0099)(66,0.0096)(67,0.0098)(68,0.0099)(69,0.0101)(70,0.0034)(72,0.0025)(73,0.0027)(74,0.003)(75,0.0016)(76,0.0018)(77,0.0018)(78,0.002)(79,0.0023)(80,0.0003)(81,0)(82,0.0004)(84,0.0008)(85,0.0009)(86,0.0007)(87,0.0004)(88,0.0002)(89,0)(90,0.0003)(91,0.0001)(92,0.0002)(94,0.0004)(95,0.0007)(96,0.0008)(97,0.0011)(98,0.001)(99,0)};
\addplot[color=black!100,dotted,very thick,mark=none]coordinates {
(1,0.0004)(2,0.0004)(3,0.0011)(4,0.0026)(5,0.0041)(6,0.0051)(7,0.0057)(8,0.009)(9,0.0092)(10,0.0125)(11,0.0129)(12,0.0165)(13,0.0169)(14,0.0175)(15,0.0271)(16,0.0322)(17,0.0333)(18,0.0346)(19,0.0356)(20,0.0581)(21,0.0591)(22,0.0599)(23,0.0611)(24,0.0657)(25,0.0783)(26,0.0786)(27,0.0796)(28,0.0827)(29,0.0833)(30,0.0958)(31,0.0963)(32,0.1043)(33,0.1059)(34,0.107)(35,0.1307)(36,0.1378)(37,0.1435)(38,0.1553)(39,0.156)(40,0.1409)(41,0.1405)(42,0.1381)(43,0.1386)(44,0.1371)(45,0.1131)(46,0.1124)(47,0.1121)(48,0.1043)(49,0.1043)(50,0.0531)(51,0.0528)(52,0.0518)(53,0.0515)(54,0.0511)(55,0.0381)(56,0.038)(57,0.0381)(58,0.0379)(59,0.0379)(60,0.0101)(61,0.0101)(62,0.0101)(63,0.0102)(64,0.0102)(65,0.0068)(66,0.0065)(67,0.0066)(68,0.0065)(69,0.0067)(70,0.0028)(72,0.0028)(73,0.0029)(74,0.0031)(75,0.0022)(76,0.0023)(77,0.0022)(78,0.0022)(79,0.0024)(80,0.0003)(81,0.0001)(82,0.0001)(84,0.0007)(85,0.0008)(86,0.0007)(87,0.0005)(88,0.0004)(89,0.0003)(90,0.0005)(91,0.0003)(92,0.0003)(94,0.0001)(95,0)(96,0.0002)(97,0.0003)(98,0.0003)(99,0)};
\addplot[black!10,dashed] coordinates{(40,0)(40,0.22)};
\legend{TE,DE};
\end{axis}
\end{tikzpicture}
\caption{Total and direct effects in the Adult dataset with respect to gender for binarised target hours per week as a function of the binarisation threshold.}\label{fig:cont}
\end{figure}

\subsection{Non-Binary Features}
The fairness descriptors in Eqs.~\eqref{eq:tv}--\eqref{eq:ie} are defined for a binary feature $X$ and a specific order between its two states $x_0$ and $x_1$, where we regard $x_0$ as the \emph{protected} state and $x_1$ as the \emph{non-protected} one. When $X$ is not binary, we might have a whole set of protected states denoted as $\mathcal{X}_0$, with the non-protected states  consequently denoted as $\mathcal{X}_1 \coloneq \mathcal{X}\setminus \mathcal{X}_0$. The fairness descriptors definitions can then be extended by averaging over all pairs of states drawn from the two sets. For instance, for the DE in Eq.~\eqref{eq:de}:
\begin{equation}\label{eq:de3}
\mathrm{DE}_{\mathcal{X}_0,\mathcal{X}_1}(Y)\coloneq
\frac{1}{|\mathcal{X}_0| |\mathcal{X}_1|}
\sum_{\substack{x_0 \in \mathcal{X}_0,\\x_1 \in \mathcal{X}_1}}
\left[ P(Y_{x_1,W_{x_0}}=y)-P(Y_{x_0,W_{x_0}}=y) \right]
\,.
\end{equation}
Such an arithmetic average can be replaced by a weighted average taking into account the relative distribution of the protected and non-protected states inside their groups, i.e.,
\begin{equation}\label{eq:de4}
\mathrm{DE}_{\mathcal{X}_0,\mathcal{X}_1}(Y)\coloneq
\sum_{\substack{x_0 \in \mathcal{X}_0,\\x_1 \in \mathcal{X}_1}}
\frac{n(x_0)}{n(\mathcal{X}_0)}
\frac{n(x_1)}{n(\mathcal{X}_1)}
\left[ P(Y_{x_1,W_{x_0}}=y)-P(Y_{x_0,W_{x_0}}=y) \right]
\,,
\end{equation}
where $n(x)$ denote the marginal count of $X=x$. In practice, care should be taken when interpreting this average. If the different terms in the sum have different signs, cancellations may occur, potentially masking important disparities. In such cases, it may be preferable to complement the aggregate measure with a more fine-grained analysis of the individual pairwise effects. Being based on linear operations, averages like the ones in Eq.~\eqref{eq:de3} and Eq.~\eqref{eq:de4} also preserve the decompositions we have discussed in the previous section.

\begin{example}\label{ex:adult_categorical_race_3w^2z}
{\it Let us demonstrate the averages in Eq.~\eqref{eq:de3} and \eqref{eq:de4} by means of an example based on the same dataset of Ex.~\ref{ex:adult_te_decomp}. As a protected feature for the target \emph{income} here we consider the \emph{race}, for which we distinguish between the protected set $\mathcal{X}_0$ with states \emph{African-American} and \emph{Indigenous} and the other non-protected ones, namely \emph{Caucasian}, \emph{Asian-American}, and \emph{Others}. The model has two confounders (\emph{age} and \emph{native country}) and three mediators (\emph{hours}, \emph{occupation} and \emph{relationship}). Tab.~\ref{tab:averages} depicts the pairwise TE and DE for all the six combinations together with the corresponding averages computed as in Eq.~\eqref{eq:de3} and \eqref{eq:de4}. Indeed, in the example, the effects have variable signs, and cancellations occur.}
\end{example}

\begin{table}[htp!]
\centering
\begin{tabular}{p{4cm}p{4cm}rp{0.5cm}r}
\hline
 $x_0$&$x_1$&TE&&DE\\
\hline
 Caucasian&African-American&-0.040&&-0.070\\
 Caucasian&Indigenous&-0.126&&-0.126\\
 Asian-American&African-American&0.092&&0.030\\
 Asian-American&Indigenous&0.005&&-0.004\\
 Others&African-American&0.095&&0.027\\
 Others&Indigenous&0.008&&-0.002\\
\hline
$\mathcal{X}_0$&$\mathcal{X}_1$&0.006&&-0.024\\
$\mathcal{X}_0$&$\mathcal{X}_1$&0.007&&-0.011\\
\hline
\end{tabular}
\caption{Pairwise effects on the \emph{income} in the \emph{Adult} dataset. The last line displays averages computed, respectively, as in Eq.~\eqref{eq:de3} and Eq.~\eqref{eq:de4}.\label{tab:averages}}
\end{table}

As a special case of non-binary protected features, let us consider the ordinal case. Say, for instance, that the states are ordered as $\mathcal{X} \coloneq (x_0,\ldots,x_{n})$. A \emph{threshold} state $x_t$, $0<t<n$, might induce a partition between protected and non-protected states, say $\mathcal{X}_0 \coloneq \{x_0,\ldots,x_t\}$. The fairness descriptors are consequently computed as in Eq.~\eqref{eq:de3}. Such averages might discard information on how effects evolve along the ordered scale of the states of $X$. For a deeper analysis we suggest instead a \emph{stepwise} decomposition:
\begin{equation}\label{eq:stepwise}
\mathrm{TE}_{x_0,x_{n-1}}(y)=\sum_{k=0}^{n-1}
\mathrm{TE}_{x_k,x_{k+1}}(y)\,.
\end{equation}
The above decomposition trivially holds for the TE defined as in Eq.~\eqref{eq:te}, as well as for TV and SE. For DE and IE, a similar decomposition only holds in the linear case. The following example is intended to show how the stepwise decomposition might reveal useful information when coping with ordinal protected features.

\begin{example}
{\it As in Ex.~\ref{ex:adult_te_decomp}, we consider the \emph{Adult} dataset, where the Boolean target variable $Y$ indicates whether an individual’s annual income exceeds \$50{,}000. As the protected feature $X$, we consider \emph{education}, whose levels exhibit a natural ordinal structure ranging from high-school to master’s degree. The TE decomposition in Eq.~\eqref{eq:stepwise}, illustrated for these data in Fig.~\ref{fig:stepwise}, shows, for instance, that the substantial decrease in the probability of earning a high income when comparing high-school dropouts with people holding a master remains relevant even when the comparison is moved to higher levels, such as junior college dropouts.}
\end{example}

This example demonstrates the value of stepwise causal decomposition in identifying where along an ordered social process disparities are most visible, complementing traditional binary fairness analysis.

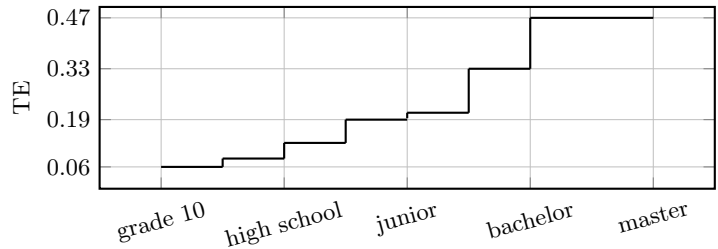
\begin{figure}[htp!]
\centering
\begin{tikzpicture}[]
\begin{axis}[
height=4cm,
width=0.8\textwidth,
grid=both,
grid style={line width=.1pt, draw=gray!50},
xlabel={},          
ylabel={TE},
xmin=15,xmax=25,
ymin=0,ymax=0.5,
ytick={0.06,0.19,0.33,0.47},
yticklabels={0.06,0.19,0.33,0.47},
xtick={16,18,20,22,24},
xticklabel style = {rotate = 15},
xticklabels={grade 10,high school,junior,bachelor,master}]
\addplot[mark=none,domain=16:17]{0.06};
\addplot[mark=none,domain=17:18]{0.083};
\addplot[mark=none,domain=18:19]{0.126};
\addplot[mark=none,domain=19:20]{0.19};
\addplot[mark=none,domain=20:21]{0.209};
\addplot[mark=none,domain=21:22]{0.33};
\addplot[mark=none,domain=22:24]{0.47};
\addplot +[solid,black,mark=none] coordinates {(17, 0.06) (17, 0.083)};
\addplot +[solid,black,mark=none] coordinates {(18, 0.083) (18, 0.126)};
\addplot +[solid,black,mark=none] coordinates {(19, 0.126) (19, 0.19)};
\addplot +[solid,black,mark=none] coordinates {(20, 0.194) (20, 0.209)};
\addplot +[solid,black,mark=none] coordinates {(21, 0.209) (21, 0.33)};
\addplot +[solid,black,mark=none] coordinates {(22, 0.33) (22, 0.47)};
\end{axis}
\end{tikzpicture}
\caption{Income TE for increasing educational levels in the \emph{Adult} dataset.}\label{fig:stepwise}
\end{figure}

\section{Fairness Analysis and LLM Reports}\label{sec:app-ui-llm}
As shown by the examples discussed in this paper, the fairness descriptors and the decomposition relations we derived in Sects.~\ref{sec:sfm} and \ref{sec:theory_new} can be very effective in highlighting discriminatory mechanisms underlying the input data. However, a joint interpretation of the results can be challenging, especially when different effects have opposite signs and cancellations occur. Data as the ones in Tab.~\ref{tab:adult} may not be immediately interpretable for a non-specialist reader. A Sankey diagram as in Fig.~\ref{fig:adult_sankey_decomp_ex1} might offer an appealing summary, but this advantage holds only when all effects share the same sign.

These interpretability challenges suggest the possibility of leveraging LLMs as reporting agents. In our framework, the identification formulae are first used to compute the fairness descriptors. The software implementation of this part is a library, we called \textsc{FairMind}. The resulting structured outputs are then provided to the LLM, which acts as a post-processing interpretability layer and generates a coherent, human-readable analytical report. Alternatively, the LLM may generate reports directly from the input data, provided that appropriate contextual information is included in the prompt.

\begin{figure}[htp!]
\centering
\begin{tikzpicture}[]
\tikzstyle{block2} = [draw,text centered,rounded corners]
\tikzstyle{block} = [rectangle,draw,text width=17em,rounded corners,minimum height=17em,anchor=north]
\node [block2] at (0,1) (DB) {Dataset on $X$, $Y$, $Z$, $W$};
\node [block2] at (6,1) (Effects) {TV,TE,SE,DE,IE};
\node [block] at (0,0) (Report2) 
{\scriptsize {\bf Report \# 1 (Pure LLM)}\\Male students show a higher probability than females of pass ({\color{red}{difference +0.0665}}). The decomposition indicates this gap is primarily driven by the direct effect of gender on pass, while {\color{red}{the mediator study time contributes negligibly and slightly negatively (indirect)}}, and there is no spurious/confounding contribution in the provided data. In substantive terms, the observed advantage for male over female in passing in this dataset is mainly attributable to direct pathways rather than study time or measured confounding.};
\node [block] at (6,0) (Report) {\scriptsize {\bf Report \# 2 (\textsc{FairMind} + LLM)}\\Males have a higher probability of pass than females by approximately {\color{green!50!black}{$0.0626$}}. The disparity is primarily driven by the direct effect ($\text{DE} = 0.1119$), with a smaller mediated contribution through study time ($\text{IE} = 0.0492$); spurious/confounding effects are zero. Thus, most of the advantage in passing the exam for males over females appears to operate via direct pathways rather than via differences in study time, {\color{green!50!black}{although study time accounts for a non-negligible portion}} of the gap.};
\path [draw,->] (DB) -- (Effects) node [pos=0.5, above, sloped] {\textsc{FairMind}};
\path [draw,->] (Effects) -- (Report) node [pos=0.5, left] {LLM};
\path [draw,->] (DB) -- (Report2) node [pos=0.5, right] {LLM};
\end{tikzpicture}
\caption{Comparing a report directly generated from the data by LLM against the one obtained by first processing the data with the \textsc{FairMind} tools. The protected binary feature is \emph{gender} (\emph{male} and \emph{female}), while the target is the \emph{passing} of the exam. For the sake of simplicity, no confounders are considered, and the mediator is the \emph{study time} with four ordinal states.}
\label{fig:reports}
\end{figure}
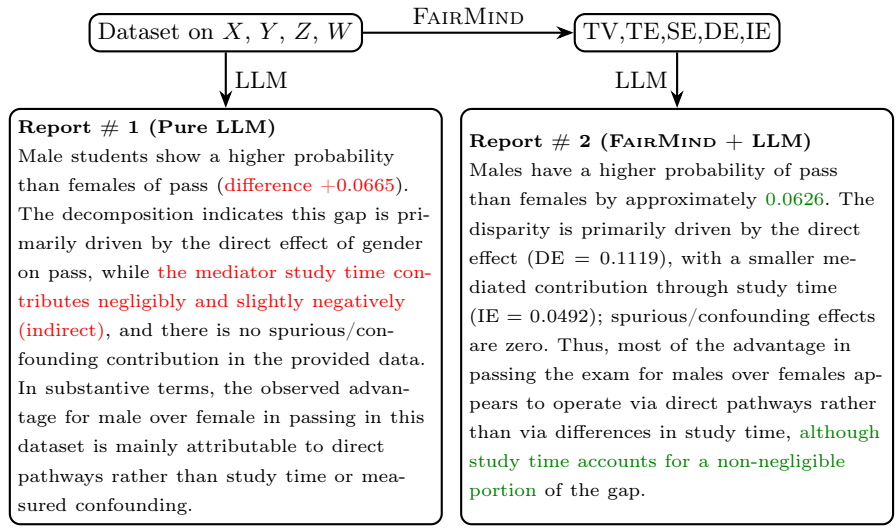

\paragraph{\bf Limits of the Pure-LLM Approach.} 
Fig.~\ref{fig:reports} shows a comparison between our approach and the \emph{pure-LLM} alternative for the \emph{Student-Mat} dataset considered also in Ex.~\ref{ex:student}. The differences in the conclusions of the two reports are highlighted with colours. As widely discussed in the literature \cite{song2025survey}, the reasoning abilities of LLMs remain limited, which may lead to incorrect conclusions---for instance, the statement regarding the contribution of the mediator in the pure-LLM report. Similarly, the pure-LLM approach may result in inaccurate computations (e.g., the incorrect TE calculation). Beyond these critical issues, it is also important to note that asking an LLM to directly process raw data is highly inefficient. In our example, we were able to process only a small dataset (395 instances), whereas applying the same approach to larger datasets (thousands of instances or more) proved infeasible. Even with this small dataset, the computational advantages of our method are evident: our approach required 5,673 tokens, while the pure-LLM approach required 15,150 tokens. Finally, it is worth noting that datasets potentially affected by fairness issues are also likely to raise \emph{privacy} concerns. Asking an LLM to process raw data may therefore pose additional privacy risks. In contrast, \textsc{FairMind} can perform the computations within a secure environment and then provide the LLM only with the numerical outputs of the analysis, which is considerably less critical from a privacy perspective. For these reasons, we regard \textsc{FairMind} as a novel and robust AutoML tool for fairness analysis, designed to facilitate the adoption of fairness assessment among ML practitioners.

\paragraph{\bf Code and Benchmarks.}
Our \textsc{FairMind} framework is based on a Python implementation of all the identification and decomposition formulas discussed in this paper combined with preprocessing utilities and a user interface. We adhere to the SFM: the user only needs to provide, in addition to the data, a specification of the target variable, the protected feature, and the mediators and confounders. The input data are treated as endogenous observations. These data are used to train a Bayesian network over the specified variables, with a topology determined by the SFM assumptions and with directed arcs from the confounders to the protected attribute. For parameter estimation, we employ Laplace smoothing with pseudo-count $\alpha=1$ across all states of all variables. The resulting Bayesian network is then used to compute the observational queries required by our identification formulae. The \texttt{pgmpy} \cite{ankanpgmpy2024} library is used for Bayesian network inference.


\paragraph{\bf Prompt.} 
 The outputs produced by \textsc{FairMind} serve as the input to an LLM reporting module. To generate structured fairness reports, we adopt a dual-prompt architecture composed of a system prompt and a user prompt. The system prompt defines the role of the model as a statistician and causal fairness expert, while simultaneously imposing strict output constraints, formatting rules and structural requirements. The user prompt contains only the structured JSON object with pre-computed fairness decomposition results. This separation ensures that statistical computation remains fully deterministic and external to the LLM, which is instead exclusively responsible for semantic interpretation and narrative synthesis. We use \emph{gpt-05-mini} as the reporting model, balancing interpretative quality with computational efficiency and set the reasoning effort to \emph{high}. The design of the system prompt, also attached in the supplementary material, stimulates structured reasoning while avoiding unconstrained free-form generation.
 
The prompt (see App.~\ref{app:prompt}) enforces a predefined analytical structure. It prescribes the exact organization of the report, the order in which the effects must be described, how effect decompositions must be interpreted and how extreme values across subgroups must be reported. This aligns with the \emph{Chain of Thought} paradigm \cite{wangUnderstandingChainofThoughtPrompting2023}, whose effectiveness relies primarily on maintaining relevance to the query and a coherent ordering of reasoning steps, rather than on the strict logical validity of intermediate demonstrations. Accordingly, our prompt architecture enforces structural coherence and strict grounding in the provided JSON data. By constraining the model to use only the supplied variables and numerical values, we mitigate hallucination risks and ensure that the generated report remains semantically anchored to the computed fairness results.

\section{Conclusions}\label{sec:conc}
The present work should be regarded as one of the first attempts to promote the adoption of fairness analysis among machine learning and data science practitioners. We adopt an AutoML perspective and leverage the identifiability results derived by Plečko and Bareinboim \cite{pleckoCausalFairnessAnalysis2024} for their standard fairness model. This enables the development of a software tool capable of performing such analyses with minimal human intervention. Finally, the use of LLMs as reporting agents in the last stage of the pipeline allows the generation of outputs that are easily interpretable even by non-specialists.

Our approach can also be interpreted from a \emph{neuro-symbolic} perspective, as a successful example of integration and cooperation between LLM-based (neural) agents and high-level (symbolic) tools responsible for preliminary data analysis and the required evaluations. Our discussion of the advantages of this approach, compared with the direct use of an LLM in a purely end-to-end manner, can be readily extended to many other tasks.

Several directions should be considered for future work. First, the tool could incorporate structural learning algorithms in order to automatically distinguish between mediators and confounders while also removing irrelevant features. In addition, the LLM component could be further improved, for instance by adopting prompt optimisation techniques. Further challenging directions include relaxing some of the assumptions of the SFM framework, which may lead to partially identifiable queries, as well as developing automated methods to mitigate fairness issues detected by the tool.
\newpage
\bibliographystyle{splncs04}
\bibliography{biblio}

@article{surveyAutoMLLLMs,
    author = {Gu, Yang and You, Hengyu and Cao, Jian and Yu, Muran and Fan, Haoran and Qian, Shiyou},
    title = {Large Language Models for Constructing and Optimizing Machine Learning Workflows: A Survey},
    year = {2025},
    publisher = {Association for Computing Machinery},
    journal = {ACM Trans. Softw. Eng. Methodol.}
}

@article{barbudo2023,
	author = {Barbudo, Rafael and Ventura, Sebasti{\'a}n and Romero, Jos{\'e}Ra{\'u}l},
	journal = {Knowledge and Information Systems},
	number = {12},
	pages = {5097--5149},
	title = {Eight years of {{AutoML}}: categorisation, review and trends},
	volume = {65},
	year = {2023}}

@inproceedings{song2025survey,
  title={A survey on large language model reasoning failures},
  author={Song, Peiyang and Han, Pengrui and Goodman, Noah},
  booktitle={2nd AI for Math Workshop (ICML 2025)},
  year={2025}}

@inproceedings{cortez2008,
  author={Cortez, P. and Silva, Alice},
  title={Using data mining to predict secondary school student performance},
    booktitle = "Proceedings of 5th Annual Future Business Technology Conference" ,
  year={2008}}

@book{pearlCausalityModelsReasoning2009,
  title = {Causality: {{Models}}, {{Reasoning}}, and {{Inference}}},
  shorttitle = {Causality},
  author = {Pearl, Judea},
  year = 2009,
  publisher = {Cambridge University Press}}

@book{pearlCausalInferenceStatistics2016,
  title = {Causal Inference in Statistics: A Primer},
  shorttitle = {Causal Inference in Statistics},
  author = {Pearl, Judea and Glymour, Madelyn and Jewell, Nicholas P.},
  year = 2016,
  publisher = {Wiley},
}

@article{pleckoCausalFairnessAnalysis2024,
  title = {Causal {{Fairness Analysis}}: {{A Causal Toolkit}} for {{Fair Machine Learning}}},
  shorttitle = {Causal {{Fairness Analysis}}},
  author = {Ple{\v c}ko, Drago and Bareinboim, Elias},
  year = 2024,
  journal = {Foundations and Trends in Machine Learning},
  volume = {17},
  number = {3},
  pages = {304--589},
}

@inproceedings{mahajanCausalFairnessInActionOpenSource2025,
  title = {{{CausalFairnessInAction}}: {{An}} Open Source {P}ython Library for Causal Fairness Analysis},
  shorttitle = {{{CausalFairnessInAction}}},
  booktitle = {Proceedings of the {{NeurIPS}} 2025 {{Workshop}} on {{Causal Fairness}}},
  author = {Mahajan, Kriti},
  year = 2025
}

@article{weertsCanFairnessBe2024,
  title = {Can {{Fairness}} Be {{Automated}}? {{Guidelines}} and {{Opportunities}} for {{Fairness-aware AutoML}}},
  shorttitle = {Can {{Fairness}} Be {{Automated}}?},
  author = {Weerts, Hilde and Pfisterer, Florian and Feurer, Matthias and Eggensperger, Katharina and Bergman, Edward and Awad, Noor and Vanschoren, Joaquin and Pechenizkiy, Mykola and Bischl, Bernd and Hutter, Frank},
  year = 2024,
  journal = {Journal of Artificial Intelligence Research},
  volume = {79},
  pages = {639--677},
}

@inproceedings{dingRetiringAdultNew2021,
  title = {Retiring Adult: New Datasets for Fair Machine Learning},
  shorttitle = {Retiring Adult},
  booktitle = {Proceedings of the 35th {{International Conference}} on {{Neural Information Processing Systems}}},
  author = {Ding, Frances and Hardt, Moritz and Miller, John and Schmidt, Ludwig},
  year = 2021,
  pages = {6478--6490},
  publisher = {Curran},
}

@inproceedings{wangUnderstandingChainofThoughtPrompting2023,
  title = {Towards {{Understanding Chain-of-Thought Prompting}}: {{An Empirical Study}} of {{What Matters}}},
  shorttitle = {Towards {{Understanding Chain-of-Thought Prompting}}},
  booktitle = {Proceedings of the 61st {{Annual Meeting}} of the {{Association}} for {{Computational Linguistics}} ({{Volume}} 1: {{Long Papers}})},
  author = {Wang, Boshi and Min, Sewon and Deng, Xiang and Shen, Jiaming and Wu, You and Zettlemoyer, Luke and Sun, Huan},
  editor = {Rogers, Anna and {Boyd-Graber}, Jordan and Okazaki, Naoaki},
  year = 2023,
  pages = {2717--2739},
  publisher = {Association for Computational Linguistics},
}

@article{gallesAxiomaticCharacterizationCausal1998,
  title = {An {{Axiomatic Characterization}} of {{Causal Counterfactuals}}},
  author = {Galles, David and Pearl, Judea},
  year = 1998,
  journal = {Foundations of Science},
  volume = {3},
  number = {1},
  pages = {151--182},
}

@inproceedings{correaCounterfactualGraphicalModels2025,
  title = 	 {Counterfactual Graphical Models: Constraints and Inference},
  author =       {Correa, Juan D. and Bareinboim, Elias},
  booktitle = 	 {Proceedings of the 42nd International Conference on Machine Learning},
  pages = 	 {11245--11254},
  year = 	 {2025},
  volume = 	 {267},
  publisher =    {PMLR}}

@inproceedings{pearlDirectIndirectEffects2001,
  title = {Direct and Indirect Effects},
  booktitle = {Proceedings of the {{Seventeenth}} Conference on {{Uncertainty}} in Artificial Intelligence},
  author = {Pearl, Judea},
  year = 2001,
  pages = {411--420},
  publisher = {Morgan Kaufmann},
  }

@article{ankanpgmpy2024,
  author  = {Ankur Ankan and Johannes Textor},
  title   = {{P}gmpy: A {P}ython Toolkit for {B}ayesian Networks},
  journal = {Journal of Machine Learning Research},
  year    = {2024},
  volume  = {25},
  number  = {265},
  pages   = {1--8},
}

@inproceedings{correaNestedCounterfactualIdentification2021,
author = {Correa, Juan D. and Lee, Sanghack and Bareinboim, Elias},
title = {Nested counterfactual identification from arbitrary surrogate experiments},
year = {2021},
isbn = {9781713845393},
publisher = {Curran Associates Inc.},
address = {Red Hook, NY, USA},
booktitle = {Proceedings of the 35th International Conference on Neural Information Processing Systems},
articleno = {525},
numpages = {12},
series = {NIPS '21}
}

\newpage
\appendix
\section{Proofs}\label{app:proofs} 
The formulae in Prop.~\ref{prop:identification} were previously derived by \cite{pleckoCausalFairnessAnalysis2024}. We provide here alternative derivations, partially based on the recently proposed \emph{ctf-calculus} \cite{correaCounterfactualGraphicalModels2025}.
\noindent{\bf Proof of Proposition~\ref{prop:identification}}.
{\it We first prove Eq.~\eqref{eq:te2}. By total probability:
\begin{equation}
P(Y_{x} = y) = \sum_{z} P(Y_x=y \mid z) P(z)\,.
\end{equation}
Then, as $Y_x \dsep X \mid Z$, we obtain:
\begin{equation}
P(Y_{x} = y)= \sum_{z} P(Y_x=y \mid x,z) P(z)\,,
\end{equation}
and finally by the consistency axiom \cite{gallesAxiomaticCharacterizationCausal1998}:
\begin{equation}
P(Y_{x} = y) =  \sum_{z} P(Y= y \mid x,z) P(z) \,.
\end{equation} 
As we assume that the SCM generates the data, $P(\cdot)=\hat{P}(\cdot)$ and hence, from Eq.~\eqref{eq:te}, Eq.~\eqref{eq:te2}. We similarly proceed for Eq.~\eqref{eq:de2}. First, by total probability:
\begin{equation}
P(Y_{x_1,W_{x_0}}=y)= \sum_{z} P(Y_{x_1,W_{x_0}}=y \mid z) P(z)\,,
\end{equation}
and, again by total probability:
\begin{equation}
P(Y_{x_1,W_{x_0}}=y)=\sum_{z,w} P(Y_{x_1,W_{x_0}}=y \mid z, W_{x_0}=w) P(W_{x_0}=w \mid z) P(z)\,.
\end{equation}
The next steps are based on the first two rules of the \emph{ctf-calculus} \cite{correaCounterfactualGraphicalModels2025}. The first rule implies $W_{x_0}=w \Rightarrow Y_{x_1,W_{x_0}} = Y_{x_1,w}$ and hence:
\begin{equation}
P(Y_{x_1,W_{x_0}}=y)=\sum_{z,w} P(Y_{x_1,w}=y \mid z, W_{x_0}=w) P(W_{x_0}=w \mid z) P(z) \,,
\end{equation}
then by the second rule we have $Y_{x_1,w} \dsep W_{x_0} \mid Z$, and hence:
\begin{equation}
P(Y_{x_1,W_{x_0}}=y)=\sum_{z,w} P(Y_{x_1,w}=y \mid z) P(W_{x_0}=w \mid z) P(z)\,.
\end{equation}
We similarly proceed for the next steps:
\begin{align} 
P(Y_{x_1,W_{x_0}}=y)
    =&& \sum_{z,w} P(Y_{x_1,w}=y \mid x_1, w, z) P(W_{x_0}=w \mid z) P(z)\,,\\
    =&& \sum_{z,w} P(Y=y \mid x_1, w, z) P(W_{x_0}=w \mid z) P(z)\,,\\
    =&& \sum_{z,w} P(Y=y \mid x_1, w, z) P(W=w \mid x_0, z) P(z)\,,
\end{align}
and hence Eq.~\eqref{eq:de2}. Eq.~\eqref{eq:ie2} is derived analogously. \qed}

Again for the sake of completeness, let us detail also the derivation of the decomposition in Prop.~\ref{prop:total_Eff_decomp}.

\noindent{\bf Proof of Proposition \ref{prop:total_Eff_decomp}.}
{\it By the \emph{composition axiom} \cite{gallesAxiomaticCharacterizationCausal1998}, 
the TE defined as in Eq.~\eqref{eq:te} rewrites as:
\begin{equation}
    \mathrm{TE}_{x_0,x_1}(y) = P(Y_{x_1, W_{x_1}}=y) - P(Y_{x_0, W_{x_0}}=y)\,.
\end{equation}
By introducing the zero-sum term $P(Y_{x_1,W_{x_0}}=y)$, we have that $\mathrm{TE}_{x_0,x_1}(y)$ rewrites as: 
\begin{equation}
P(Y_{x_1, W_{x_1}}=y) - P(Y_{x_0, W_{x_0}}=y) + P(Y_{x_1, W_{x_0}}=y) - P(Y_{x_1, W_{x_0}}=y)\,,
\end{equation}
i.e., Eq.~\eqref{eq:dec_te}. In linear SCMs, IE is antisymmetric and this implies Eq.~\eqref{eq:dec_te2}.}\qed

\noindent{\bf Proof of Theorem~\ref{thm:multiple_mediators}.}
{\it Without loss of generality, consider two mediators only, ordered as $W^1, W^2$. We can write the contribution of $W^1$ for the IE by providing an identification formula following from the proof of Prop.~\ref{prop:ie} and the conditional independence between the mediators (i.e., $W^1 \dsep W^2 \mid X, Z$) as follows:
\begin{align}
&P\big(y_{x_0,(W^1)_{x_1},(W^2)_{x_0}}\big) 
- P\big(y_{x_0,(W^1)_{x_0},(W^2)_{x_0}}\big)
 \nonumber \\ 
&= \sum_{z} P(z) \sum_{w^1,w^2} P(y \mid x_0, w^1, w^2, z) \Big(P(w^1 \mid x_1, z) - P(w^1 \mid x_0, z)\Big) P(w^2 \mid x_0, z)\,.
\label{eq:order1}
\end{align}
Because of the linearity assumption, we have the following decomposition:
\begin{equation}\label{eq:decomp}
P(y \mid x_0, w^1, w^2, z)=\phi_1(w^1,z)+\phi_2(w^2,z)\,,
\end{equation}
where $\phi_1$ and $\phi_2$ are linear functions describing the contribution
of $W^1$ and $W^2$, respectively. Eq.~\eqref{eq:decomp} allows to rewrite the right-hand side of Eq.~\eqref{eq:order1} as:
\begin{equation}\label{eq:aa}
\sum_z P(z)\sum_{w^1,w^2}
\big(\phi_1(w^1,z)+\phi_2(w^2,z)\big)
\big(P(w^1|x_1,z)-P(w^1|x_0,z)\big)
P(w^2\mid x_0,z)\,.
\end{equation}
By simple algebra, we rewrite Eq.~\eqref{eq:aa} as:
\begin{align}
&\sum_z P(z)\Big(
\sum_{w^1}\phi_1(w^1,z)\big(P(w^1\mid x_1,z)-P(w^1\mid x_0,z)\big)
\sum_{w^2}P(w^2\mid x_0,z) \nonumber\\
&\quad+\sum_{w^2}\phi_2(w^2,z)P(w^2\mid x_0,z)
\sum_{w^1}\big(P(w^1\mid x_1,z)-P(w^1\mid x_0,z)\big)
\Big)\,.
\label{eq:order1mid}
\end{align}
Finally, as $\sum_{w^2}P(w^2\mid x_0,z)=1$ and
$\sum_{w^1}(P(w^1\mid x_1,z)-P(w^1\mid x_0,z))=0$, we obtain:
\begin{align}
\sum_z P(z)\sum_{w^1}\phi_1(w^1,z)\big(P(w^1\mid x_1,z)-P(w^1\mid x_0,z)\big)\,.
\label{eq:order1final}
\end{align}
Thus, given the $(W^1, W^2)$, the contribution reduces strictly to an expression depending on $W^1$ only. With the inverse order the expression still depends on $W^1$ only. Hence, the individual contribution of each mediator is isolated from the others and does not depend on the ordering. The same argument extends straightforwardly to an arbitrary number of mediators.}\qed

The decomposition in Prop.~\ref{prop:ie} can be seen as an instance of Th. 6.6 in  \cite{pleckoCausalFairnessAnalysis2024}. Yet, we report here a self-contained derivation for the sake of completeness.

\noindent{\bf Proof of Proposition~\ref{prop:ie}.}
{\it
Let $A_i \coloneq P\big(y_{x_0,(\bm{W}^{\leq i})_{x_1},(\bm{W}^{>i})_{x_0}}\big)$ denote the $i$-th term appearing in Eq.~\eqref{eq:iew}. The summation in Eq.~\eqref{eq:multipleeffW} forms a telescoping sum, i.e.:
\begin{equation}
    \sum_{i=1}^k A_i - A_{i-1} = A_k - A_0\,.
\end{equation}
For $i=k$, $\bm{W}^{\leq k} = \bm{W}$ and $\bm{W}^{>k} = \emptyset$, thus $A_k = P(y_{x_0,\bm{W}_{x_1}})$. For $i=0$, $\bm{W}^{\leq 0} = \emptyset$ and $\bm{W}^{>0} = \bm{W}$, thus $A_0 = P(y_{x_0,\bm{W}_{x_0}})$. Therefore, $A_k - A_0 = P(y_{x_0,\bm{W}_{x_1}}) - P(y_{x_0, \bm{W}_{x_0}})$, which recover the indirect effect definition in Eq.~\eqref{eq:ie}, concluding the decomposition.

We are left to identify the individual term $A_i$. We proceed to apply the \emph{Counterfactual Unnesting Theorem} \cite{correaNestedCounterfactualIdentification2021} and total probability over $\bm{Z}$, thus,
\begin{equation}
    A_i = \sum_{\bm z, \bm w} P\big(y_{x_0, \bm{w}}, \bm{W}^{\leq i}_{x_1} = \bm{w}^{\leq i}, \bm{W}^{> i}_{x_0} = \bm{w}^{> i} \mid \bm{z} \big) P(\bm{z}) \label{eq:proof_decomp_ie_1}\,.
\end{equation}
The subsequent steps follow from the \emph{ctf-calculus} \cite{correaCounterfactualGraphicalModels2025}. Throughout the proof, we construct the Ancestral Multi-World Network $\mathcal{G}_A$---where the input causal graph $\mathcal{G}$ is the SFM in Fig.\ref{fig:causalgraph_detailed} generalised to $k$ ordered mediators---which encodes the counterfactual independences necessary to apply Rule 2.

Consequently, we observe that $Y_{x_0,\bm{w}} \dsep \{\bm{W}^{\leq i}_{x_1},\bm{W}^{>i}_{x_0}\} \mid \bm{Z}$ in $\mathcal{G}_A$, which, in turn, enables us to apply the Rule 2: 
\begin{equation}
\label{eq:proof_decomp_ie_2} 
    A_i = \sum_{\bm z, \bm w} P\big(y_{x_0,\bm{w}} \mid \bm{z} \big) P\big(\bm{W}^{\leq i}_{x_1} = \bm{w}^{\leq i}, \bm{W}^{> i}_{x_0} = \bm{w}^{> i} \mid \bm{z}\big) P(\bm{z}) \,.
\end{equation}
For the first factor of Eq.~\eqref{eq:proof_decomp_ie_2}, we apply, respectively, Rule 2 because $Y_{x_0,\bm{w}} \dsep \{X, \bm{W} \}\mid \bm{Z}$ in $\mathcal{G}_A$ and Rule 1, i.e. $(X=x_0,\bm{W}=\bm{w}) \Rightarrow Y_{x_0, \bm{w}}=Y$, this yields:
\begin{equation}
   P\big(y_{x_0,\bm{w}} \mid \bm{z} \big) = P\big(y_{x_0, \bm{w}} \mid x_0, \bm{w}, \bm{z} \big) = P\big(y \mid x_0, \bm{w}, \bm{z} \big) \label{eq:proof_decomp_ie_first_term_obs}\,. 
\end{equation}
Let us now focus on the second factor of Eq.~\eqref{eq:proof_decomp_ie_2}. Since the mediators are evaluated in a fixed topological order, for $j\in \{1, \dots, k\}$, any mediator $W^j$ only depends on its topological predecessors $\bm{W}^{<j}$ (in addition to $X, \bm{Z}$). Let us define the $j$-th intervention as:
\begin{equation}
    \label{eq:proof_decomp_ie_jth_int}
    x^{(j)} \coloneq 
    \begin{cases}
        x_1 \quad j\leq i\,,\\
        x_0 \quad j > i \,.
    \end{cases}
\end{equation}
Note that second factor of Eq.~\eqref{eq:proof_decomp_ie_2} can be expressed using its set elements as $P(W^1_{x^{(1)}} = w^1, \dots, W^k_{x^{(k)}} = w^k \mid \bm{z})$. Using the chain rule of probability and Eq.~\eqref{eq:proof_decomp_ie_jth_int}, we obtain the factorisation:
\begin{align}
    P\big(W^1_{x^{(1)}} &= w^1, \dots, W^k_{x^{(k)}} = w^k \mid \bm{z}\big) \nonumber\\
       &= \prod_{j=1}^k P\big(W^j_{x^{(j)}}=w^j\mid \{W^m_{x^{(m)}} = w^m\}_{0<m<j}, \bm{z}\big) \label{eq:proof_decomp_ie_3}\,.
\end{align}
For each individual factor $j$ in the r.h.s. of Eq.~\eqref{eq:proof_decomp_ie_3}, we obtain:
\begin{align}
    P\big(W^j_{x^{(j)}} &= w^j \mid \{W^m_{x^{(m)}} = w^m\}_{0<m<j}, \bm{z} \big) \nonumber\\
    &= P\big(W^j_{x^{(j)}} = w^j \mid x^{(j)}, \bm{W}^{<j}=\bm{w}^{<j}, \{W^m_{x^{(m)}} = w^m\}_{0<m<j}, \bm{z}\big) \label{eq:proof_decom_ie_4}\\
    &= P\big(W^j_{x^{(j)}} = w^j \mid x^{(j)}, \bm{W}^{<j}=\bm{w}^{<j}, \bm{z}\big) \label{eq:proof_decom_ie_5}\\
    &= P\big(w^j \mid x^{(j)}, \bm{w}^{<j}, \bm{z}\big) \label{eq:proof_decom_ie_6} \,.
\end{align}
The derivation goes as follows. First, in Eq.~\eqref{eq:proof_decom_ie_4}, we apply Rule 2 because $W^j_{x^{(j)}} \dsep \{X, \bm{W}^{<j}\}$ when conditioning on $\{W^1_{x^{(1)}}, \dots, W^{j-1}_{x^{(j-1)}}, \bm{Z}\}$ in $\mathcal{G}_A$, which, in turn, allow us to introduce the factual evidence. Next, in Eq.~\eqref{eq:proof_decom_ie_5}, we note that once the factual precedessors $\bm{W}^{<j}$ and the treatment $X$ are explicitly observed, the only non-instantiated parameter of structural equation associated to $W^j$ is the exogenous variable $U_{W^{j}}$, making the set $\{W^m_{x^{(m)}} = w^m\}_{m<j}$ redundant. Finally, in Eq.~\eqref{eq:proof_decom_ie_6} we apply Rule 1, i.e., $X=x^{(j)} \Rightarrow W^j_{x^{(j)}} = W^j$.

Consequently, by Eq.~\eqref{eq:proof_decomp_ie_jth_int}, we can rewrite the term from Eq.~\eqref{eq:proof_decomp_ie_3} as a factorisation of observable quantities, i.e., 
\begin{align}
    P\big(W^1_{x^{(1)}} = w^1, &\dots, W^k_{x^{(k)}}=w^k \mid \bm{z}\big) \nonumber \\
    & = \prod_{j=1}^k P\big(w^j \mid x^{(j)}, \bm{w}^{<j}, \bm{z}\big) \\
    &= \prod_{j=1}^{i} P\big(w^j \mid x_1, \bm{w}^{<j}, \bm{z}\big)\prod_{j=i+1}^k P\big(w^j \mid x_0, \bm{w}^{<j}, \bm{z}\big) \label{eq:proof_decomp_ie_first_scnd_obs}\,.
\end{align}
We substitute back into Eq.~\eqref{eq:proof_decomp_ie_2} the observational expressions obtained in Eqs.~\eqref{eq:proof_decomp_ie_first_term_obs} and~\eqref{eq:proof_decomp_ie_first_scnd_obs}. Provided that the observational data is generated by an SCM compatible with the SFM, we substitute $P$ with $\hat{P}$. By marginalising out $\bm{z}$ and observing that $\bm{w}^{<j} = \bm{w}^{\leq j-1}$, we recover Eq.~\eqref{eq:id_sfm_multimediator}; concluding the proof.
\qed
}

\vspace{2ex}
We refer the reader to \cite{pleckoCausalFairnessAnalysis2024} for proof of Prop.~\ref{prop:spur_more}.

\newpage
\section{LLM Instructions Used for Fairness Report Generation}\label{app:prompt}
\begin{tcolorbox}[title=Prompt, colback=gray!5,breakable]
\begin{lstlisting}[basicstyle=\scriptsize]
## SYSTEM PROMPT (Markdown, token-efficient)

You are a statistician and causal fairness expert. You will receive fairness decomposition results as JSON.

## Output format (MUST follow exactly)
Return **one** response containing **exactly two** top-level sections in this order and with these exact labels:

1) `TEXT:`
2) `LATEX:`

- `TEXT` must be plain English with **no LaTeX**.
- `LATEX` must be a **complete standalone LaTeX document** starting with `\documentclass` and containing **only valid LaTeX** (no Markdown, no backticks, no commentary).

Use only information present in the JSON. Do **not** invent numbers or fields. If something required is missing, explicitly say it was not provided.

## Required document structure (applies to BOTH TEXT and LATEX)
### TEXT structure
- Start with: `Title: "Fairness Decomposition Report"`
- Then include these subsections as headings using the **exact names** below:
  1. `Overview of the Fairness Analysis`
  2. `Decomposition of Effects`
  3. `Stepwise Effects Across Ordered Levels of X` (conditional; see below)
- End with a conclusion-style recap (short paragraph).
- The recap MUST clearly state which group (x0 or x1, using actual group names from JSON) has higher or lower outcome probability/mean.
- It MUST explicitly describe the direction of disparity in plain words (e.g., "Females have a lower probability than Males of Y").
- It MUST explain whether the disparity is mainly driven by the direct, indirect, or spurious component.
- Do not only restate numeric values; interpret them substantively.

### LATEX structure
- Same content as TEXT, formatted as LaTeX.
- Use `\subsection*{...}` for subsection headings, with titles exactly:
  - `Overview of the Fairness Analysis`
  - `Decomposition of Effects`
  - `Stepwise Effects Across Ordered Levels of X` (only if included in TEXT)
- End with a recap section (e.g., `\subsection*{Recap}`) whose content corresponds to the TEXT recap.

## Subsection requirements

### 1) Overview of the Fairness Analysis
Write one short paragraph (3-6 sentences) covering:
- what the fairness analysis is about,
- which groups are compared (**$X=x0$ vs $X=x1$**, using JSON values when available) e.g. x0 female vs. x1 male,
- what outcome `Y` represents (from JSON),
- what was decomposed.

#### HARD RULE for continuous outcomes (strict)
If the JSON indicates **Y is continuous** AND **threshold/curve data is present**, you MUST include the following elements **in this exact order** within this subsection:

1) **Curve trend summary across thresholds**: describe whether total variation and each effect (total, direct, indirect, spurious) are constant or vary with threshold; whether they increase/decrease/non-monotonic as threshold increases; and which component changes most (if determinable).

2) **Threshold implication statement**: explain whether choosing a higher vs lower threshold would increase or decrease the observed disparity, and which effect(s) drive that change.

3) **Selected threshold statement**: explicitly state the threshold value selected by the user (exact value from JSON). If it is relatively high/low within the provided threshold grid, say so.

If curve/threshold data is missing, you MUST explicitly state that the trend cannot be assessed.

### 2) Decomposition of Effects
Provide a **bullet list** focusing on:
- total variation between the two groups,
- total effect,
- indirect effect:
  - interpret indirect effect as using baseline **x1** while the others use baseline **x0**,
  - if multiple mediators exist, describe mediator-specific decomposition using the mediator variable names,
- direct effect,
- spurious effect:
  - if multiple confounders exist, describe confounder-specific decomposition using the confounder variable names.

For each bullet, briefly interpret the effect in terms of fairness and potential discrimination. Always refer to mediators/confounders by **their variable names**; avoid generic terms.

#### X-specific and Z-specific extrema reporting (REQUIRED if present in JSON)
- **X-specific**:
  - Report the `X_value` with the **largest and smallest** direct, indirect, and total effect.
  - If not determinable: `X-specific results not provided or insufficient to determine extrema.`
  - If missing/empty: `X-specific results not provided.`
- **Z-specific (z_specific)**:
  - Same rules:
    - If not determinable: `Z-specific results not provided or insufficient to determine extrema.`
    - If missing/empty: `Z-specific results not provided.`

Do not provide generic summaries; tie statements to actual numeric results and named variables when present.

### 3) Stepwise Effects Across Ordered Levels of X (CONDITIONAL)
Include this subsection **only if**:
- `stepwise.enabled` is true **and**
- `effects_by_step` exists and is non-empty.

If included:
- State the 'X' order used (lowest -> highest).
- Summarize how effects change across adjacent steps.
- Mention any steps with small sample size (`n_rows`) or missing values.

If either condition is false or missing:
- Do **not** include this subsection,
- do **not** write its title,
- do **not** mention stepwise effects anywhere.

Keep language concise and data-scientist friendly. Do not invent numbers; if missing, state it was not provided.

## Numeric/notation rules (TEXT and LATEX)
- Clearly distinguish between **x0** and **x1** (use exact X values from JSON when available).
- Round all reported numeric results to **no more than 4 decimal digits**.

## Strict LaTeX rules (LATEX section only)
- LATEX must be **only** valid LaTeX.
- Must start with `\documentclass` and include `\begin{document}` ... `\end{document}`.
- Must include in the preamble exactly:
  - `\usepackage{geometry}`
  - `\geometry{margin=1in}`
  - `\usepackage{amsmath}`
- **All** math expressions must be inside `$...$`.
- Inequalities must use LaTeX symbols `\geq`, `\leq`, `>`, `<`.
  - Raw `>=`, `<=`, `==` are forbidden anywhere in LATEX.
  - Example: render ">=50K" as `"$\geq 50K$"`.

- Do NOT use \text{">50K"} or similar string-wrapped thresholds inside math.
- If the JSON contains threshold-style labels (e.g., ">50K", ">=50K"),
  convert them into proper mathematical notation such as:
  "$Y > 50K$" or "$Y \geq 50K$".
## Internal workflow (do not print)
1) Parse JSON: x0/x1 groups, Y type, mediators/confounders, threshold/curve data + selected threshold, stepwise flags + effects_by_step.
2) Decide which subsections apply.
3) Extract key numbers; round to <=3 decimals.
4) Write TEXT with the exact structure and required statements.
5) Convert the same content into a valid LaTeX document obeying all LaTeX rules.

Now produce the final answer with exactly:
- `TEXT:`
Text here...
- `LATEX:`
All the latex here..
---

## USER PROMPT (input-only)
{RESULTS_JSON}
\end{lstlisting}
\end{tcolorbox}
\end{document}